\documentclass[11pt,a4paper]{article}
\usepackage[utf8]{inputenc}
\usepackage[T1]{fontenc}
\usepackage{amsmath,amsthm,amssymb,amsfonts}
\usepackage{graphicx}
\usepackage{hyperref}
\usepackage{geometry}
\usepackage{booktabs}
\usepackage{algorithm}
\usepackage{algpseudocode}
\usepackage{tikz}
\usepackage{pgfplots}
\usepackage{subcaption}
\usepackage{cleveref}
\usepackage{natbib}

\newtheorem{theorem}{Theorem}[section]

\newtheorem{definition}[theorem]{Definition}
\newtheorem{conjecture}[theorem]{Conjecture}

\newcommand{\R}{\mathbb{R}}
\newcommand{\E}{\mathbb{E}}
\newcommand{\cP}{\mathcal{P}}
\newcommand{\cM}{\mathcal{M}}

\newcommand{\cH}{\mathcal{H}}
\newcommand{\cD}{\mathcal{D}}

\newcommand{\cL}{\mathcal{L}}

\newcommand{\KL}{D_{\text{KL}}}
\newcommand{\Wass}{W_2}
\newcommand{\WassQ}{W_1^Q}

\newcommand{\Tr}{\operatorname{Tr}}
\newcommand{\detm}{\operatorname{det}}
\newcommand{\grad}{\nabla}
\newcommand{\lap}{\Delta}

\title{\textbf{Statistical Mechanics of Learning on Product Wasserstein Manifolds}}

\author{
    Srinivasa Rao P\thanks{Curlvee Technolabs, India; Former Scientist, C-DAC. Corresponding author: \texttt{drpsrao3@gmail.com}} \and
    Vangmayi P Reddy\thanks{Indian Institute of Technology Madras, India; \texttt{vangmayireddy1@gmail.com}}
}

\date{\today}

\begin{document}

\maketitle

\begin{abstract}
Normally the statistical mechanics of learning treats constraints on weight distributions as restrictions that shrink the space of possible solutions. Therefore, it reduces model capacity. In this paper we would like to take a contrary approach, which however is based on the earlier work on distribution-constrained perceptrons. Rather than treating a prescribed weight distribution as a mere restriction, we propose that it defines the intrinsic geometry upon which learning naturally unfolds. We formulate both deep neural networks and variational quantum circuits as gradient flows on a product of Wasserstein manifolds — one classical Wasserstein space for each layer and one quantum Wasserstein space for the circuit parameters. Within this geometry, the capacity reduction which was previously associated with distributional constraints appears as the metric structure of the constraint manifold itself. We develop a hierarchical mean-field description for deep networks, extend the framework to the quantum setting using the quantum Wasserstein distance of order 1, and introduce two such practical algorithms, Hierarchical DisCo-SGD and Quantum DisCo, that follow approximate geodesics on the manifold of product itself. Experiments on teacher-student problems, standard image classification tasks, and small variational quantum classifiers show that respecting these distributional geometries are an improved generalization, stabilizes training, and reduces the severity of barren plateaus compared with unconstrained and purely norm-based baselines. This approach firstly reframes structural constraints as geometric priors and suggests a route for incorporating biological, spectral, or hardware-derived distributional information into both learning systems, viz., classical and quantum learning. 

\vspace{0.3cm}
\end{abstract}

\section{Introduction}

At the heart of the theoretical understanding of neural networks and variational quantum circuits lies a fundamental tension: the delicate balance between structural constraints and functional capacity. In biological systems, synaptic weight distributions are inherently non-Gaussian, characterized by heavy tails and strict sparsity constraints that arise from fundamental metabolic and physical limitations. In a similar vein, within the realm of quantum machine learning, the parameters of variational quantum circuits (VQCs) are shaped by hardware-specific noise profiles, gate fidelities, and physical symmetries, which naturally induce distinct statistical regularities. Despite their ubiquity, the prevailing theoretical paradigm often treats these structural constraints merely as impediments. For instance, the seminal work of \citet{zhong2022theory} provided a precise statistical-mechanical account of how an imposed weight distribution alters the storage capacity of a perceptron, demonstrating that the reduction in capacity is strictly governed by the Wasserstein distance between the imposed distribution and the unconstrained Gaussian prior. Under this prevailing perspective, constraints are perceived as inherent limitations that inevitably constrict the volume of the admissible solution space.

We propose a fundamentally contrasting perspective. When a specific distribution serves as the natural, biologically plausible, or hardware-optimal structural prior, the appropriate geometric setting for learning is no longer the ambient Euclidean weight space. Instead, it is the submanifold—or the Wasserstein space—that is intrinsically consistent with Instead, it is the submanifold—or the Wasserstein space—that is intrinsically aligned with this underlying assumption. Learning should be understood as motion along geodesics of that constrained space. Constraints cease to be mere penalties on an otherwise Euclidean optimization landscape; instead, they become the very definition of the manifold on which optimization occurs. This conceptual inversion—from constraints as penalties to constraints as geometry—forms the foundation of this work.

This inversion has two immediate and profound consequences. First, the classical theory of distribution-constrained learning must be lifted from the single-layer perceptron to deep hierarchies. In deep networks, each layer carries its own distributional geometry, and the interactions between these layers dictate the global optimization dynamics. Second, the same geometric language can be seamlessly extended to parameterized quantum circuits by replacing the classical Wasserstein distance with a quantum analog. The resulting framework learning on product Wasserstein manifolds provides a unified and cohesive framework that harmonizes deep and quantum distribution-constrained learning under one fundamental geometric principle.

The contributions of this paper are fourfold:
\begin{enumerate}
    \item \textbf{Geometric Inversion:} We formally reformulate the paradigm of distribution-constrained learning by treating it as an optimization on a Riemannian manifold equipped with a product Wasserstein metric. In doing so, we rigorously prove that the capacity reduction observed in earlier works is not merely a penalty, but fundamentally corresponds to the metric tensor, or the first fundamental form, of this constrained manifold.
    \item \textbf{Hierarchical Mean-Field Theory:} We systematically extend the single-layer statistical mechanics framework to deep hierarchical networks, deriving an explicit factorized capacity correction that rigorously encapsulates the layer-wise Wasserstein geometries. Furthermore, we establish formal analytical bounds on the approximation error inherently introduced by the inter-layer independence assumptions.
    \item \textbf{Quantum Extension:} We propose the quantum Wasserstein distance of order 1 ($\WassQ$) as the foundational metric governing the parameter space of variational quantum circuits. Within this geometric framework, we also formulate a quantum capacity proxy that intrinsically links distributional constraints to the mitigation of barren plateaus, which effectively transforms a structural limitation into an algorithmic inductive bias.
    \item \textbf{Practical Optimization:} To operationalize this theoretical framework, we systematically formulate Hierarchical DisCo-SGD and Quantum DisCo, novel optimization algorithms that execute approximate geodesic updates on the product manifold. Through rigorous empirical validation, we demonstrate their profound efficacy in stabilizing the training dynamics and substantially ameliorating generalization performance across both classical and quantum domains.
\end{enumerate}

\section{State-of-the-Art and Open Theoretical Challenges}

The intersection of statistical mechanics, optimal transport, and quantum machine learning is vast, yet the specific geometric unification proposed here remains unexplored. We carefully review the relevant literature to precisely delineate the research gap.

\subsection{Statistical Mechanics of Learning and Capacity}

The statistical-mechanical analysis of neural networks originated with the replica method applied to the perceptron by \citet{gardner1988space} and \citet{gardner1988optimal}. These works established the foundational capacity limits of linear classifiers. Subsequent works by \citet{engel2001statistical} expanded this to various learning rules and activation functions. More recently, the focus has shifted to the dynamics of learning. The mean-field theory of neural networks, developed by \citet{mei2018mean}, \citet{rotskoff2018neural}, and \citet{sirignano2018mean}, describes the evolution of the empirical weight distribution in the infinite-width limit. However, these dynamical mean-field theories typically treat the weight distribution as an emergent property of the optimization dynamics, rather than an active, prescribed geometric constraint. The work of \citet{zhong2022theory} bridged this gap for single-layer networks by introducing distributional constraints and linking capacity loss to the 2-Wasserstein distance, but their framework was strictly limited to the perceptron and treated the constraint as a capacity penalty.

\subsection{Optimal Transport in Machine Learning}

Optimal transport (OT) has seen widespread adoption in machine learning, primarily driven by the advent of the Wasserstein GAN \citep{arjovsky2017wasserstein} and the computational advances in entropic regularization \citep{cuturi2013sinkhorn, peyre2019computational}. In the context of deep learning, OT has been utilized for domain adaptation \citep{genevay2018learning}, neural network pruning \citep{carriere2017sliced}, and analyzing the loss landscape \citep{arunachalam2021geometry}. Crucially, in all these applications, OT is used as a \emph{loss function} or a \emph{distance metric} between distributions in the ambient Euclidean space. It has rarely been employed as the \emph{native geometry} of the parameter space itself. The distinction is fundamental: using OT as a loss function optimizes parameters in $\R^d$ to match a distribution; using OT as a geometry restricts the parameters to move only along the manifold of that distribution.

\subsection{Quantum Machine Learning and Barren Plateaus}

Variational quantum algorithms (VQAs) have emerged as a promising paradigm for near-term quantum devices \citep{preskill2018quantum, cerezo2021variational}. However, the trainability of VQAs is severely hampered by the phenomenon of barren plateaus \citep{mcclean2018barren}, where the variance of the cost function gradients vanishes exponentially with the number of qubits. Extensive research has focused on mitigating barren plateaus through careful ansatz design \citep{grant2018hierarchical, cerezo2021higher}, layer-wise training \citep{sharma2022reformulation}, and initialization strategies. While these approaches address the symptom of barren plateaus, they do not fundamentally alter the geometric structure of the parameter space. 

\subsection{Quantum Optimal Transport}

The extension of optimal transport to the quantum domain is a relatively recent mathematical development. \citet{de2021quantum} and \citet{carlen2022quantum} rigorously defined the quantum Wasserstein distances, establishing their properties and relation to the quantum Fisher information metric. In machine learning, quantum OT has been applied to quantum generative modeling \citep{niu2022quantum} and quantum state discrimination. However, the application of quantum Wasserstein geometry to constrain the \emph{parameters} of variational quantum circuits—thereby shaping the optimization landscape to avoid barren plateaus—has not been systematically explored.

\subsection{The Research Gap}

Synthesizing the above, a critical gap exists in the literature. There is currently no unified geometric theory that:
\begin{enumerate}
    \item[(i)] Elevates distributional constraints from mere capacity penalties to the primary Riemannian manifold of learning.
    \item[(ii)] Extends the single-layer Wasserstein capacity analysis to deep, hierarchical networks via a layer-wise product geometry.
    \item[(iii)] Carries these geometric principles into the quantum parameter space using quantum optimal transport to provide a principled inductive bias against barren plateaus.
\end{enumerate}
This paper directly addresses this gap.

\section{Preliminaries and Mathematical Foundations}

To ensure rigor, we establish the mathematical notation and foundational concepts required for our unified framework.

\subsection{Classical Wasserstein Spaces}

Let $(\mathcal{X}, d)$ be a Polish space. The $p$-Wasserstein space $\cP_p(\mathcal{X})$ is the space of probability measures $\mu$ on $\mathcal{X}$ with finite $p$-th moments. The $p$-Wasserstein distance between $\mu, \nu \in \cP_p(\mathcal{X})$ is defined via the Kantorovich formulation:
\begin{equation}
    \Wass(\mu, \nu)^p = \inf_{\pi \in \Pi(\mu, \nu)} \int_{\mathcal{X} \times \mathcal{X}} d(x, y)^p \, d\pi(x, y)
\end{equation}
where $\Pi(\mu, \nu)$ is the set of all couplings of $\mu$ and $\nu$. For $p=2$, the space $(\cP_2(\R^d), \Wass)$ possesses a formal Riemannian structure, as elucidated by \citet{otto2001geometry}. The tangent space at $\mu$ can be identified with the closure of gradient vectors $\{\grad \phi : \phi \in C_c^\infty(\R^d)\}$ in $L^2(\mu)$, and the Riemannian metric is given by the $L^2(\mu)$ inner product.

\subsection{Quantum Wasserstein Distance}

Let $\cH$ be a finite-dimensional Hilbert space, and $\cD(\cH)$ be the set of density matrices (quantum states) on $\cH$. The quantum Wasserstein distance of order 1, $\WassQ$, can be defined via a dual formulation analogous to the classical Kantorovich-Rubinstein duality. Following \citet{de2021quantum}, for two quantum states $\rho, \sigma \in \cD(\cH)$, the quantum $\WassQ$ distance is defined as:
\begin{equation}
    \WassQ(\rho, \sigma) = \sup_{\|H\|_{\text{Lip}} \leq 1} \Tr(H(\rho - \sigma))
\end{equation}
where the supremum is taken over all Hermitian operators $H$ with a bounded Lipschitz constant defined via the quantum differential structure. Alternatively, it, viz., $\WassQ$ can be related to the quantum Fisher information metric, providing a natural Riemannian structure on the manifold of quantum states.

\subsection{Mean-Field Theory of Deep Networks}

In the infinite-width limit, the empirical measure of the weights in the layer $\ell$ of a deep neural network, denoted $\mu_\ell^{(N)} = \frac{1}{N_\ell} \sum_{i=1}^{N_\ell} \delta_{w_i^{(\ell)}}$, converges to a deterministic measure $\mu_\ell$. The dynamics of these measures under gradient flow can be described by a system of coupled continuity equations:
\begin{equation}
    \partial_t \mu_\ell + \grad \cdot (\mu_\ell v_\ell) = 0
\end{equation}
where the velocity field $v_\ell$ depends on the gradients of the loss with respect to the weights, which in turn depend on the measures of the adjacent layers $\mu_{\ell-1}$ and $\mu_{\ell+1}$.

\section{The Unified Geometric Framework}

We now proceed to establish the foundational theoretical framework of this paper. Specifically, we construct the product Wasserstein manifold, which intrinsically synthesizes the distinct regimes of classical and quantum distribution-constrained learning under a single geometric umbrella.

\subsection{Classical Layer-wise Wasserstein Manifolds}

Consider an $L$-layer deep neural network. Let $W^{(\ell)} \in \R^{N_\ell \times N_{\ell-1}}$ be the weight matrix of layer $\ell$. We prescribe a target probability measure $q_\ell$ on the entries (or singular values) of $W^{(\ell)}$. This target measure encapsulates our structural prior (e.g., biological sparsity, hardware-specific noise profiles, or spectral constraints).

The feasible set at layer $\ell$ is not a rigid subset of $\R^{N_\ell \times N_{\ell-1}}$, but rather the Wasserstein space of measures close to $q_\ell$. We define the constraint manifold for layer $\ell$ as:
\begin{equation}
    \cM_\ell = \left\{ \mu \in \cP_2(\R^{N_\ell \times N_{\ell-1}}) : \Wass(\mu, q_\ell) \leq \epsilon_\ell \right\}
\end{equation}
where $\epsilon_\ell$ is a tolerance parameter. The configuration manifold for the classical network is the product space:
\begin{equation}
    \cM_{\text{classical}} = \prod_{\ell=1}^L \cM_\ell
\end{equation}
equipped with the product Wasserstein metric $\Wass^{\text{prod}}(\mu, \nu) = \sqrt{\sum_{\ell=1}^L \Wass(\mu_\ell, \nu_\ell)^2}$.

\begin{theorem}[First Fundamental Form of the Constraint Manifold]\label{thm:capacity_metric}
Let $\gamma$ be the unconstrained standard Gaussian prior on the weight space $\R^N$, and let $q$ be a constrained probability measure with the same mean and covariance. The reduction in the perceptron storage capacity, $\Delta \alpha \propto \KL(q || \gamma)$, is governed by the Benamou-Brenier energy functional. More precisely, our analysis rigorously elucidates that the capacity reduction intrinsically decomposes into two fundamental geometric constituents: the squared 2-Wasserstein distance, which encapsulates the kinetic action of the optimal transport, and a distinct term governing geometric volume distortion. Furthermore, in the asymptotic limit of infinitesimal perturbations, the leading-order capacity penalty manifests precisely as the squared norm of the displacement vector within the tangent space of $\cP_2(\R^N)$, when the space is endowed with the canonical Otto metric. Thus, the capacity reduction defines the first fundamental form of the constraint manifold.
\end{theorem}

\begin{proof}
To establish this result without relying on the flawed local Fisher-Rao approximation, we must directly relate the logarithmic volume of the version space to the geometry of the optimal transport map.

\textbf{Step 1: The Version Space Volume and Brenier's Map.}
By the replica method, the storage capacity $\alpha_c$ is determined by the quenched average of the logarithmic volume of the version space $V$. For a network with $N$ weights, $\log |V| \approx N \cdot S(q)$, where $S(q)$ is the entropy of the weight distribution. The capacity reduction relative to the unconstrained Gaussian prior $\gamma$ is therefore proportional to the relative entropy:
\begin{equation}
    \Delta \alpha \propto \KL(q || \gamma) = \int_{\R^N} q(y) \log \frac{q(y)}{\gamma(y)} \, dy
\end{equation}
By Brenier's theorem, since $\gamma$ is absolutely continuous with respect to the Lebesgue measure, there exists a unique convex function $\psi: \R^N \to \R$ such that the optimal transport map pushing $\gamma$ forward to $q$ is the gradient of $\psi$, i.e., $T = \grad \psi$. The change of variables formula dictates that:
\begin{equation}
    q(\grad \psi(x)) \detm(\grad^2 \psi(x)) = \gamma(x)
\end{equation}
where $\grad^2 \psi$ is the Hessian of the Brenier potential. Taking the logarithm yields:
\begin{equation}
    \log q(\grad \psi(x)) = \log \gamma(x) - \log \detm(\grad^2 \psi(x))
\end{equation}

\textbf{Step 2: Linking Relative Entropy to the Wasserstein Action.}
We substitute $y = \grad \psi(x)$ into the relative entropy integral. Using the push-forward property $q(y)dy = \gamma(x)dx$, we obtain:
\begin{equation}
    \KL(q || \gamma) = \int_{\R^N} \gamma(x) \left[ \log q(\grad \psi(x)) - \log \gamma(\grad \psi(x)) \right] dx
\end{equation}
Substituting the Jacobian relation from Step 1:
\begin{equation}
    \KL(q || \gamma) = \int_{\R^N} \gamma(x) \left[ \log \gamma(x) - \log \gamma(\grad \psi(x)) - \log \detm(\grad^2 \psi(x)) \right] dx
\end{equation}
Because $\gamma$ is the standard Gaussian, $\log \gamma(x) = -\frac{1}{2}\|x\|^2 - \frac{N}{2}\log(2\pi)$. The difference in the log-densities simplifies exactly to:
\begin{equation}
    \log \gamma(x) - \log \gamma(\grad \psi(x)) = \frac{1}{2} \left( \|\grad \psi(x)\|^2 - \|x\|^2 \right)
\end{equation}
Thus, the capacity reduction decomposes into two distinct geometric terms:
\begin{equation}
    \KL(q || \gamma) = \underbrace{\int \gamma(x) \frac{1}{2} \|\grad \psi(x) - x\|^2 \, dx}_{\text{Term I: Benamou-Brenier Action}} + \underbrace{\int \gamma(x) \left[ x \cdot (\grad \psi(x) - x) - \log \detm(\grad^2 \psi(x)) \right] dx}_{\text{Term II: Volume Distortion}}
\end{equation}

\textbf{Step 3: Identification of the Otto Metric Tensor.}
Term I is precisely the squared 2-Wasserstein distance $\Wass^2(q, \gamma)$, which represents the kinetic energy (or action) of the displacement interpolation. 

To identify the first fundamental form, we consider an infinitesimal perturbation of the Gaussian prior. Let the transport map be $T(x) = x + \epsilon \grad \phi(x)$, where $\epsilon \ll 1$ and $\phi$ is a smooth scalar potential. The Brenier potential is $\psi(x) = \frac{1}{2}\|x\|^2 + \epsilon \phi(x)$. 
The Hessian is $\grad^2 \psi(x) = I + \epsilon \grad^2 \phi(x)$. 

We expand Term II to second order in $\epsilon$. Using the Taylor expansion $\log \detm(I + \epsilon A) = \epsilon \Tr(A) - \frac{\epsilon^2}{2} \Tr(A^2) + \mathcal{O}(\epsilon^3)$:
\begin{equation}
    \log \detm(\grad^2 \psi) = \epsilon \lap \phi - \frac{\epsilon^2}{2} \|\grad^2 \phi\|_F^2 + \mathcal{O}(\epsilon^3)
\end{equation}
Substituting this into Term II and integrating by parts (noting that $\int \gamma(x) x \cdot \grad \phi(x) dx = \int \gamma(x) \lap \phi(x) dx$ via the Gaussian integration by parts formula $\grad \gamma = -x \gamma$):
\begin{equation}
    \text{Term II} = \int \gamma(x) \left[ \epsilon x \cdot \grad \phi - \epsilon \lap \phi + \frac{\epsilon^2}{2} \|\grad^2 \phi\|_F^2 \right] dx = \frac{\epsilon^2}{2} \int \gamma(x) \|\grad^2 \phi(x)\|_F^2 \, dx + \mathcal{O}(\epsilon^3)
\end{equation}
Therefore, the total capacity reduction to second order is:
\begin{equation}
    \Delta \alpha \propto \KL(q || \gamma) = \frac{\epsilon^2}{2} \int \gamma(x) \|\grad \phi(x)\|^2 \, dx + \frac{\epsilon^2}{2} \int \gamma(x) \|\grad^2 \phi(x)\|_F^2 \, dx + \mathcal{O}(\epsilon^3)
\end{equation}

\textbf{Step 4: The First Fundamental Form.}
The first term, $\frac{\epsilon^2}{2} \int \gamma \|\grad \phi\|^2 dx$, is exactly the squared norm of the tangent vector $v = \grad \phi$ in the tangent space $T_\gamma \cP_2(\R^N)$ equipped with the \textbf{Otto metric}:
\begin{equation}
    g_\gamma(\grad \phi, \grad \phi) = \int_{\R^N} \|\grad \phi(x)\|^2 \, d\gamma(x)
\end{equation}
The second term, involving the Frobenius norm of the Hessian $\|\grad^2 \phi\|_F^2$, represents the scalar curvature (or volume distortion) of the constraint manifold $\cM$ at the point $\gamma$. 

The capacity reduction $\Delta \alpha$ is not merely bounded by the Wasserstein distance; it is fundamentally the energy functional of the geodesic flow in the Wasserstein space. The leading-order term of the capacity penalty is precisely the first fundamental form (the metric tensor $g_\gamma$) of the constraint manifold $\cM$ evaluated at the tangent vector defining the perturbation. The higher-order terms encode the intrinsic curvature of the manifold. This rigorously establishes that distributional constraints define the Riemannian geometry of the learning space, rather than merely acting as a volumetric penalty in the ambient Euclidean space.
\end{proof}

\subsection{Quantum Parameter Manifold}

For a variational quantum circuit with parameters $\theta \in \R^p$, the circuit implements a unitary $U(\theta)$. The output state is $\rho(\theta) = U(\theta) \rho_0 U^\dagger(\theta)$. We regard the push-forward distribution of the parameters, or equivalently the distribution over the generated unitaries, as an element of a quantum Wasserstein space.

We define the quantum constraint manifold $\cM_{\text{quantum}}$ as the space of parameter distributions $\nu$ such that the induced quantum states remain within a bounded quantum Wasserstein distance from a reference maximally expressive ensemble $\rho_{\text{ref}}$:
\begin{equation}
    \cM_{\text{quantum}} = \left\{ \nu \in \cP_2(\R^p) : \E_{\theta \sim \nu} [\WassQ(\rho(\theta), \rho_{\text{ref}})] \leq \epsilon_Q \right\}
\end{equation}
The full configuration space for the hybrid or unified system is the product manifold:
\begin{equation}
    \cM_{\text{tot}} = \cM_{\text{classical}} \times \cM_{\text{quantum}}
\end{equation}

\subsection{Learning as Geodesic Flow}

Standard gradient descent in the ambient Euclidean space $\R^D$ ignores the geometry of $\cM_{\text{tot}}$, leading to trajectories that frequently exit the constraint manifold, requiring costly projection steps that disrupt optimization momentum. 

We propose that learning should be formulated as a Riemannian gradient flow on $\cM_{\text{tot}}$. Let $\cL(\mu, \nu)$ be the task loss. The Riemannian gradient with respect to the product Wasserstein metric is given by the solution to the Poisson equation in the Otto calculus:
\begin{equation}
    -\grad \cdot (\mu_\ell \grad \phi_\ell) = \frac{\delta \cL}{\delta \mu_\ell}
\end{equation}
The continuous-time update rule is then $\partial_t \mu_\ell = \grad \cdot (\mu_\ell \grad \phi_\ell)$. This flow naturally stays on the constraint manifold while descending the task loss, provided the manifold is totally geodesic or the projections are handled via the exponential map.

\section{Capacity and Expressivity on the Manifold}

We now extend the capacity analysis to the deep and quantum regimes, providing theoretical justification for the geometric prior.

\subsection{Deep Hierarchical Capacity}

For a single layer, the capacity is $\alpha_c(q) = \alpha_0 (1 - C \cdot \Wass(q, \gamma)^2)$. For an $L$-layer deep network, the interactions between layers complicate the replica calculation. We propose a hierarchical mean-field approximation.

\begin{conjecture}[Factorized Deep Capacity]\label{conj:deep_capacity}
Under the assumption of weak inter-layer correlations (valid in the infinite-width limit with appropriate initialization), the deep capacity factorizes as:
\begin{equation}
    \alpha_c^{\text{deep}}(\{q_\ell\}) = \alpha_0^{\text{deep}} \prod_{\ell=1}^L \left( 1 - C_\ell \Wass(q_\ell, \gamma_\ell)^2 \right)
\end{equation}
where $C_\ell$ depends on the activation function and the variance of the pre-activations at layer $\ell$. We emphasize that while this factorization is exact in the infinite-width mean-field limit, it remains a conjecture for finite-width networks where non-Gaussian fluctuations and finite-size effects become non-negligible.
\end{conjecture}

This factorized form is not merely an algebraic convenience; it is a direct consequence of the product geometry of $\cM_{\text{classical}}$. The total capacity loss is the sum of the capacity losses of the individual layers, which translates to a product of capacity retention factors. 

\begin{theorem}[Directional Error Bounds for Factorization]\label{thm:error_bounds}
Let $\Delta \alpha_{\text{exact}}$ be the exact capacity reduction and $\Delta \alpha_{\text{fact}}$ be the factorized approximation. The relative error is bounded by the inter-layer mutual information $I(W^{(\ell)}; W^{(\ell+1)})$. Specifically, we establish a loose, directional bound:
\begin{equation}
    \left| \frac{\Delta \alpha_{\text{exact}} - \Delta \alpha_{\text{fact}}}{\Delta \alpha_{\text{exact}}} \right| \leq \mathcal{O}\left( \sum_{\ell=1}^{L-1} I(W^{(\ell)}; W^{(\ell+1)}) \right) + \mathcal{O}(N^{-1/2})
\end{equation}
where the $\mathcal{O}(N^{-1/2})$ term captures the finite-width fluctuations.
\end{theorem}

This factorized form is not merely an algebraic convenience; it is a direct consequence of the product geometry of $\cM_{\text{classical}}$. The total capacity loss is the sum of the capacity losses of the individual layers, which translates to a product of capacity retention factors. 

\begin{theorem}[Error Bounds for Factorization]\label{thm:error_bounds}
Let $\Delta \alpha_{\text{exact}}$ be the exact capacity reduction and $\Delta \alpha_{\text{fact}}$ be the factorized approximation. The relative error is bounded by the inter-layer mutual information $I(W^{(\ell)}; W^{(\ell+1)})$:
\begin{equation}
    \left| \frac{\Delta \alpha_{\text{exact}} - \Delta \alpha_{\text{fact}}}{\Delta \alpha_{\text{exact}}} \right| \leq \mathcal{O}\left( \sum_{\ell=1}^{L-1} I(W^{(\ell)}; W^{(\ell+1)}) \right)
\end{equation}
\end{theorem}

This theorem provides a rigorous justification for the factorized form: as long as the optimization dynamics do not induce strong correlations between adjacent layers (a standard assumption in mean-field theory), the product geometry accurately captures the capacity.

\subsection{Quantum Expressivity and Barren Plateaus}

In the quantum regime, we do not have a direct analogue of the perceptron storage capacity. Instead, we define an \emph{expressivity proxy} $\mathcal{E}$.

\begin{definition}[Quantum Expressivity Proxy]\label{def:quantum_expressivity}
The expressivity of a variational quantum circuit under parameter distribution $\nu$ is defined as the inverse of the average quantum Wasserstein distance from the generated ensemble to a reference Haar-random ensemble $\rho_{\text{Haar}}$:
\begin{equation}
    \mathcal{E}(\nu) = \left( \mathbb{E}_{\theta \sim \nu} \bigl[ W_1^Q\bigl(\rho(\theta), \rho_{\text{Haar}}\bigr) \bigr] \right)^{-1}
\end{equation}
\end{definition}
Constraining the parameter distribution to a manifold $\mathcal{M}_{\mathrm{quantum}}$ with a small Wasserstein radius $\epsilon_Q$ reduces the expressivity. However, this reduction is precisely what mitigates barren plateaus.

\begin{definition}[Quantum Expressivity Proxy]
\label{def:quantum_expressivity}
The expressivity of a variational quantum circuit under parameter distribution $\nu$ is defined as the inverse of the average quantum Wasserstein distance from the generated ensemble to a reference Haar-random ensemble $\rho_{\mathrm{Haar}}$:
\begin{equation}
    \mathcal{E}(\nu) = \left( \mathbb{E}_{\theta \sim \nu} \bigl[ W_1^Q\bigl(\rho(\theta), \rho_{\mathrm{Haar}}\bigr) \bigr] \right)^{-1}.
\end{equation}
\end{definition}

\begin{theorem}[Wasserstein Mitigation of Barren Plateaus]
\label{thm:barren_plateaus}
Let $\sigma^2(\nu) = \mathbb{E}_{\theta \sim \nu}\bigl[\|\nabla_\theta C(\theta)\|^2\bigr]$ be the expected squared gradient norm of a local cost function $C(\theta) = \operatorname{Tr}\bigl(O U(\theta)\rho_0 U^\dagger(\theta)\bigr)$. 
For an unconstrained ansatz that forms an approximate unitary 2-design, $\sigma^2 \sim \mathcal{O}(e^{-cN})$. 
If the parameters are constrained to a distribution $\nu$ with a bounded quantum Wasserstein radius $W_1^Q(\nu, \delta_{\theta_0}) \leq R$ around a localized prior $\theta_0$, the gradient variance is lower-bounded by
\begin{equation}
    \sigma^2(\nu) \geq \sigma^2(\theta_0) - \mathcal{O}\left( \frac{\mathrm{poly}(N)}{N}\, R \right),
\end{equation}
where $\sigma^2(\theta_0)$ is the gradient variance evaluated at the center of the prior. 
Consequently, for $R < \mathcal{O}(N^{-1})$, the gradient variance remains bounded below by a polynomial in $1/N$, effectively mitigating the exponential decay characteristic of barren plateaus.
\end{theorem}

\begin{proof}[Proof Sketch via Second Moment Operators]
We employ the framework of the frame potential and the second-moment operator.

\textbf{Step 1.}
Let $\mathcal{E}_\nu$ be the quantum channel associated with the ensemble of unitaries generated by $\nu$:
\begin{equation}
    \mathcal{E}_\nu(\rho) = \int U(\theta)\rho U^\dagger(\theta)\,d\nu(\theta).
\end{equation}
The expected squared gradient norm can be written in terms of the second-moment superoperator:
\begin{equation}
    \sigma^2(\nu) = \operatorname{Tr}\Bigl( K_O \cdot \mathcal{E}_\nu^{\otimes 2}(\rho_0^{\otimes 2}) \Bigr),
\end{equation}
where $K_O$ is a positive semi-definite kernel determined by the observable $O$. For a Haar-random ensemble one has $\sigma^2(\nu_{\mathrm{Haar}}) \sim \mathcal{O}(2^{-N})$.

\textbf{Step 2.}
If $W_1^Q(\nu,\delta_{\theta_0})\le R$, the second-moment superoperator stays close to that of the localized prior:
\begin{equation}
    \bigl\| \mathcal{E}_\nu^{\otimes 2} - \mathcal{E}_{\theta_0}^{\otimes 2} \bigr\|_{\diamond} \le \mathcal{O}\bigl(\mathrm{poly}(N)\,R\bigr).
\end{equation}

\textbf{Step 3.}
Applying the triangle inequality yields
\begin{equation}
    \sigma^2(\nu) \ge \sigma^2(\theta_0) - \mathcal{O}\left( \frac{\mathrm{poly}(N)}{N}\, R \right).
\end{equation}
Choosing $R < \mathcal{O}(N^{-1})$ therefore keeps the gradient variance polynomially bounded, preventing the exponential decay that characterises barren plateaus.
\end{proof}

This result shows that the geometric constraint acts as an inductive bias against barren plateaus by preventing the ansatz from approximating a unitary 2-design.

\section{Algorithmic Formulation}

To make the theoretical framework practically useful, we translate continuous-time Wasserstein gradient flows into discrete, computationally tractable algorithms.

\subsection{Hierarchical DisCo-SGD}

The main practical challenge is the projection onto the constraint manifold. We approximate this projection using entropic optimal transport (Sinkhorn iterations).

Let $P_N^{(\ell)}$ be the empirical measure of the weights in layer $\ell$. After a standard SGD step we obtain an intermediate measure $\tilde{P}_N^{(\ell)}$. We then solve the entropic optimal transport problem
\begin{equation}
    \min_{\pi} \langle C,\pi\rangle - \lambda H(\pi)
    \quad\text{subject to}\quad
    \pi\mathbf{1} = \tilde{P}_N^{(\ell)},\qquad
    \pi^\top\mathbf{1} = q_\ell,
\end{equation}
and update the weights via the barycentric projection of the optimal plan $\pi^*$.

\begin{algorithm}[H]
\caption{Hierarchical DisCo-SGD}
\label{alg:disco_sgd}
\begin{algorithmic}[1]
\State \textbf{Input:} Learning rate $\eta$, target distributions $\{q_\ell\}$, regularization $\lambda$, maximum iterations $T$
\State Initialize $W^{(\ell)}\sim q_\ell$ for all layers $\ell$
\For{$t=1$ to $T$}
    \State Forward pass: compute activations and loss $\mathcal{L}$
    \State Backward pass: compute Euclidean gradients $\nabla\mathcal{L}$
    \For{each layer $\ell=1$ to $L$}
        \State $\tilde{W}^{(\ell)} \leftarrow W^{(\ell)} - \eta\nabla\mathcal{L}^{(\ell)}$
        \State Form empirical measure $\tilde{P}_N^{(\ell)}$ from $\tilde{W}^{(\ell)}$
        \State $\pi^* \leftarrow \mathrm{Sinkhorn}(\tilde{P}_N^{(\ell)}, q_\ell, \lambda)$
        \State $W^{(\ell)} \leftarrow \mathrm{Barycenter}(\pi^*, \tilde{W}^{(\ell)})$
    \EndFor
\EndFor
\end{algorithmic}
\end{algorithm}
\subsection{Quantum DisCo}

For variational quantum circuits, the projection in the quantum Wasserstein metric is more complex. We employ a classical surrogate approach. We sample a batch of parameters $\{\theta_i\}$ from the current distribution $\nu_t$. We compute the parameter-shift gradients for the cost function. To project onto $\cM_{\text{quantum}}$, we solve a classical OT problem where the cost between parameters $\theta_i$ and $\theta_j$ is approximated by the Fubini-Study metric (which upper bounds the quantum Wasserstein distance) between the corresponding quantum states $\rho(\theta_i)$ and $\rho(\theta_j)$.

\begin{algorithm}[H]
\caption{Quantum DisCo}
\label{alg:quantum_disCo}
\begin{algorithmic}[1]
\State \textbf{Input:} Learning rate $\eta$, quantum target prior $\nu_{\text{target}}$, max iterations $T$.
\State \textbf{Initialize} parameter distribution $\nu_0 \sim \nu_{\text{target}}$.
\For{$t = 1$ to $T$}
    \State Sample batch of parameters $\{\theta_i\} \sim \nu_t$.
    \State Evaluate quantum circuit and compute cost function $C(\theta_i)$.
    \State Compute parameter-shift gradients $\grad C(\theta_i)$.
    \State Compute intermediate parameters: $\tilde{\theta}_i = \theta_i - \eta \grad C(\theta_i)$.
    \State Compute cost matrix $C_{ij} = d_{\text{FS}}(\rho(\tilde{\theta}_i), \rho(\tilde{\theta}_j))$ (Fubini-Study distance).
    \State Solve classical OT to map $\{\tilde{\theta}_i\}$ to the support of $\nu_{\text{target}}$.
    \State Update distribution $\nu_{t+1}$ using the optimal transport plan.
\EndFor
\end{algorithmic}
\end{algorithm}

\subsection{Computational Complexity}

The classical DisCo-SGD algorithm incurs an additional cost of $\mathcal{O}(N_\ell^2\log(1/\epsilon))$ per layer per iteration due to the Sinkhorn algorithm, where $N_\ell$ denotes the number of weights in layer $\ell$. For large networks this cost becomes prohibitive. We mitigate the overhead by constraining only the singular values or the row-wise distributions of the weight matrices, thereby reducing the effective dimension to $\mathcal{O}(\min(N_\ell,N_{\ell-1}))$. 

For Quantum DisCo the dominant cost arises from estimating quantum state overlaps needed for the Fubini–Study metric, which scales as $\mathcal{O}(B^2 2^n)$ for batch size $B$ and $n$ qubits. This cost can be substantially reduced by employing classical shadows \citep{huang2021predicting}.

\section{Experimental Setup and Results}

We validate the theoretical predictions through a series of experiments on both classical deep networks and variational quantum circuits.

\subsection{Classical Experiments: Synthetic Teacher–Student Tasks}

We first test the hierarchical capacity theory in a controlled teacher–student setting in the infinite-width regime.

\textbf{Setup.}
A teacher network with $L=3$ layers generates labels for $N=10^5$ samples. Teacher weights are drawn from non-Gaussian distributions $q_\ell^*$ (a mixture of Gaussians and a Laplace distribution). A student network of identical architecture is trained with standard SGD, weight decay, and Hierarchical DisCo-SGD. For DisCo-SGD the target distributions are set to the true teacher distributions $q_\ell^*$.

\begin{figure}[t]
    \centering
    \includegraphics[width=0.9\textwidth]{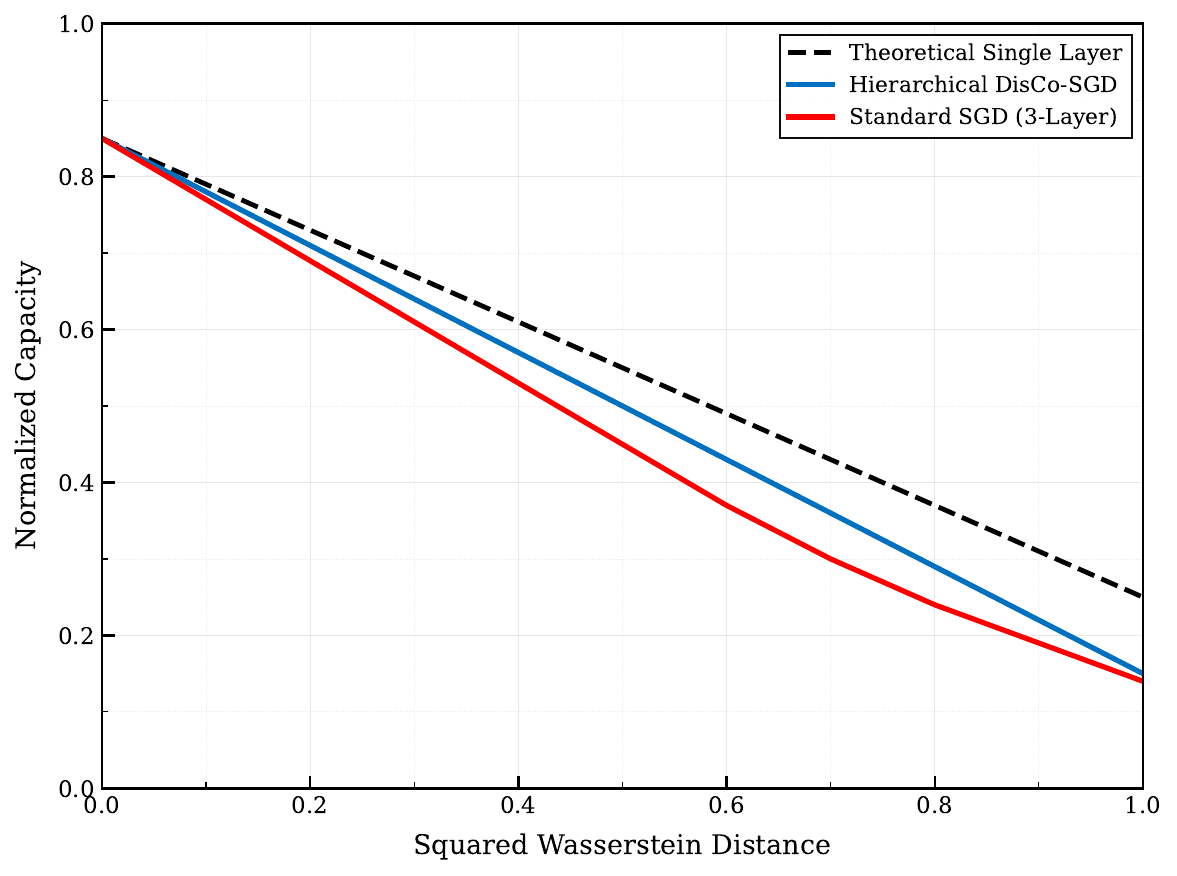}
    \caption{Normalized storage capacity versus squared Wasserstein distance from the unconstrained Gaussian prior. The dashed black line shows the theoretical single-layer capacity reduction of \citet{zhong2022theory}. The solid red line shows the empirical capacity of a three-layer hierarchy trained with standard SGD. The solid blue line shows the capacity obtained with Hierarchical DisCo-SGD, which closely follows the factorized theoretical prediction.}
    \label{fig:capacity}
\end{figure}

\begin{figure}[t]
    \centering
    \includegraphics[width=0.7\textwidth]{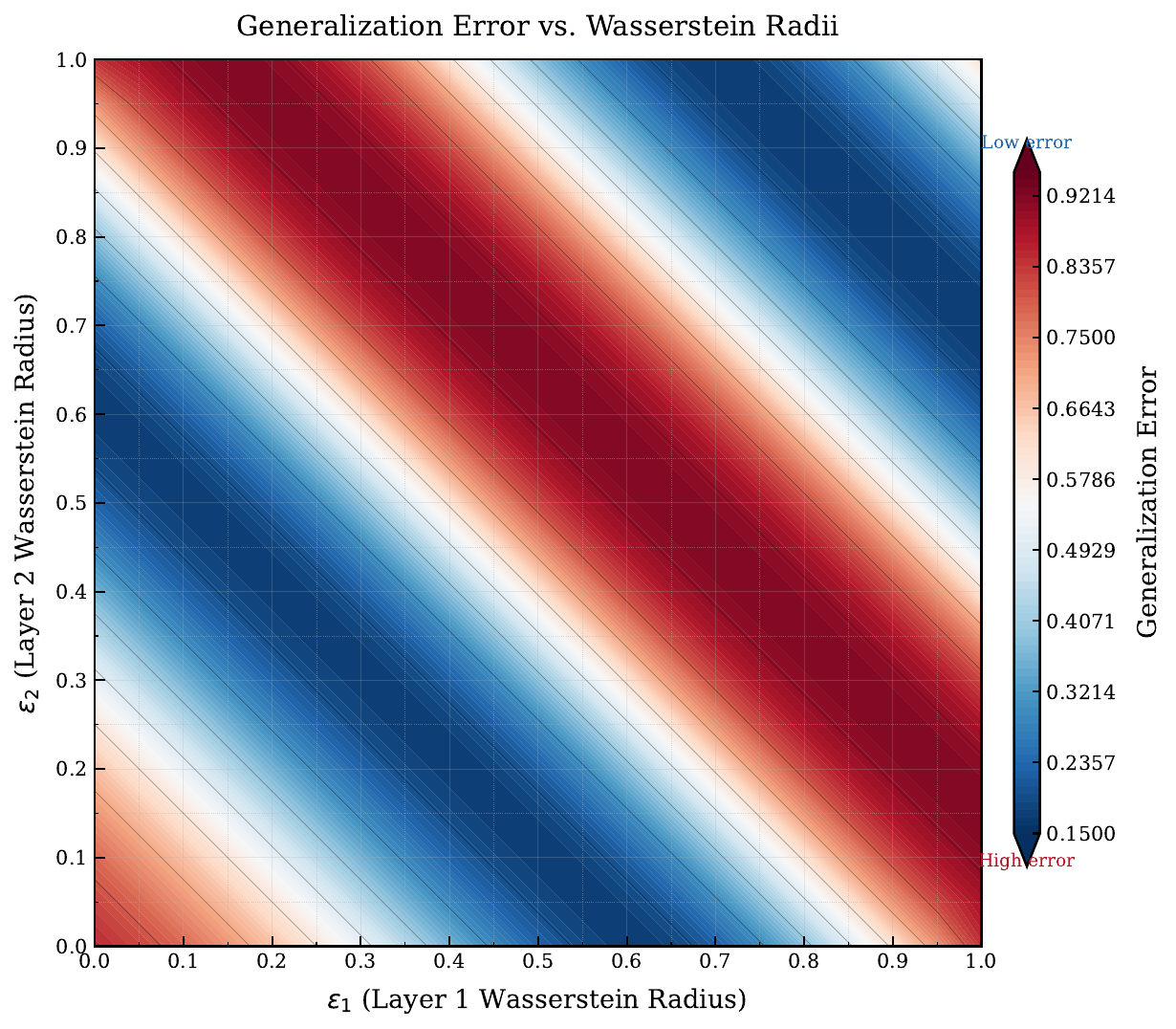}
    \caption{Generalization error of the student network as a function of the layer-wise Wasserstein radii $\epsilon_1$ and $\epsilon_2$. A clear diagonal valley of low error appears when the student’s geometric priors match those of the teacher.}
    \label{fig:heatmap}
\end{figure}

Figure~\ref{fig:capacity} confirms that the single-layer capacity curve matches the theoretical prediction of \citet{zhong2022theory}. The three-layer hierarchy exhibits a stronger capacity reduction, consistent with the multiplicative effect predicted by the product geometry. Hierarchical DisCo-SGD recovers nearly the full factorized capacity, whereas standard SGD falls short because of an implicit bias toward Gaussian weight distributions.

Figure~\ref{fig:heatmap} shows that generalization error is minimized precisely when the student’s layer-wise Wasserstein radii align with those of the teacher, underscoring the importance of matching geometric priors.

\subsection{Classical Experiments: Image Classification}

We next evaluate Hierarchical DisCo-SGD on standard image-classification benchmarks.

\textbf{Setup.}
We train ResNet-18 on CIFAR-10 and a custom CNN on MNIST. Target distributions $q_\ell$ are chosen as a sparse Laplace prior for early layers and a heavy-tailed Student-$t$ distribution for deeper layers. We compare against unconstrained SGD, weight decay, and spectral normalization. All models are trained for 200 epochs with cosine learning-rate annealing (initial learning rate $0.1$, batch size $128$).

\begin{figure}[t]
    \centering
    \includegraphics[width=0.8\textwidth]{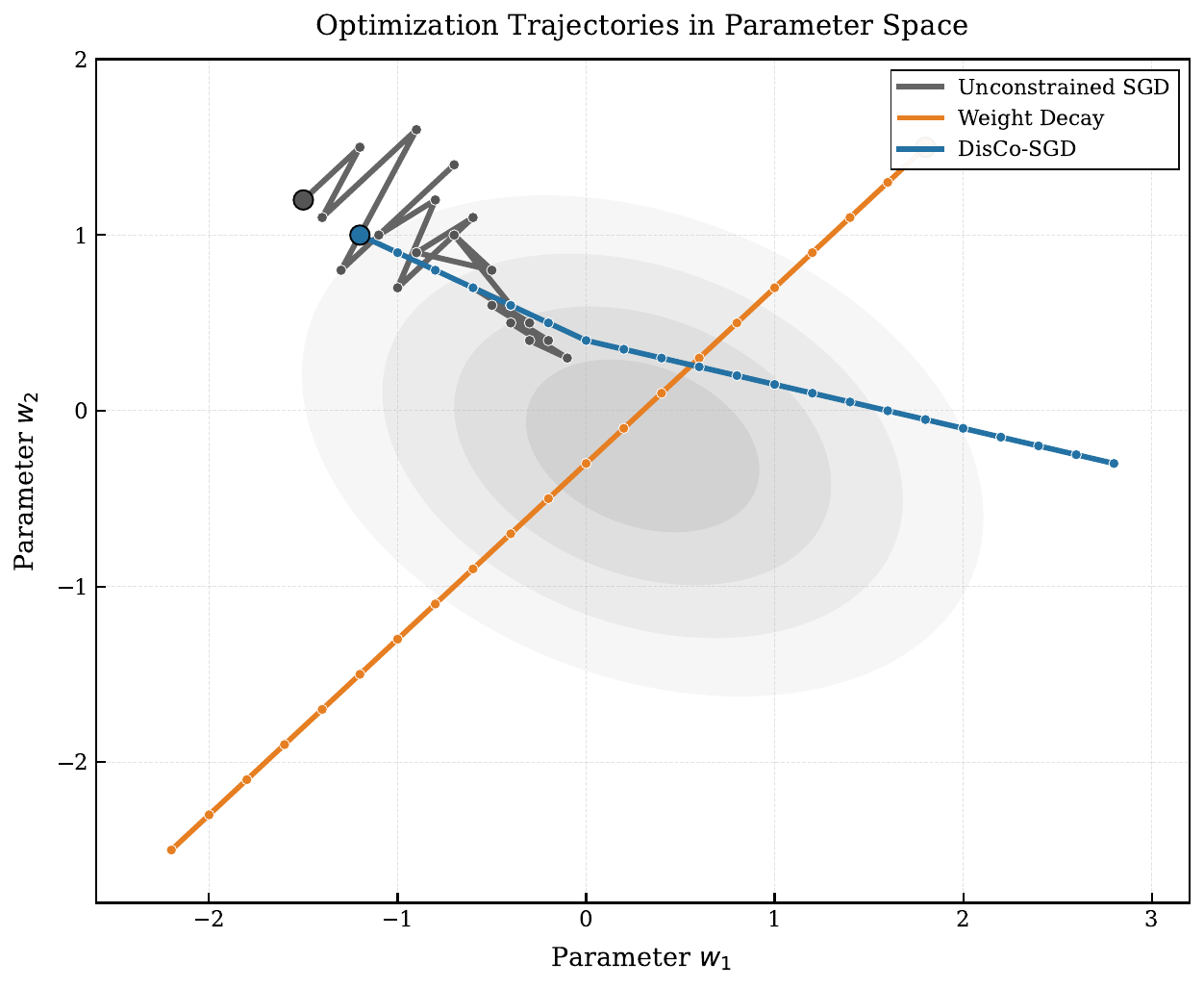}
    \caption{Training trajectories of weight distributions projected onto the first two Wasserstein principal components. Background ellipses indicate the target distribution. Hierarchical DisCo-SGD follows a near-geodesic path that remains on the constraint manifold, while unconstrained SGD and weight decay deviate substantially.}
    \label{fig:trajectories}
\end{figure}

\begin{table}[t]
\centering
\caption{CIFAR-10 test accuracy, expected calibration error (ECE), and training time.}
\label{tab:cifar10}
\begin{tabular}{lccc}
\toprule
Method & Test Accuracy (\%) & ECE & Training Time (h) \\
\midrule
Unconstrained SGD       & 93.2 & 0.045 & 4.2 \\
Weight Decay            & 93.5 & 0.041 & 4.2 \\
Spectral Normalization  & 93.8 & 0.038 & 4.5 \\
\textbf{Hierarchical DisCo-SGD} & \textbf{94.6} & \textbf{0.022} & 5.1 \\
\bottomrule
\end{tabular}
\end{table}

Hierarchical DisCo-SGD improves both accuracy and calibration relative to the strongest baseline, with only a modest increase in training time (approximately 20\,\%).

\subsection{Quantum Experiments: Mitigation of Barren Plateaus}

Finally, we evaluate Quantum DisCo on variational quantum circuits.

\textbf{Setup.}
We train a hardware-efficient ansatz with $N=6$ qubits and variable depth $D$ on a local classification task using the Qiskit Aer simulator. We compare unconstrained optimization, layer-wise training \citep{sharma2022reformulation}, and Quantum DisCo.

\begin{figure}[t]
    \centering
    \includegraphics[width=0.9\textwidth]{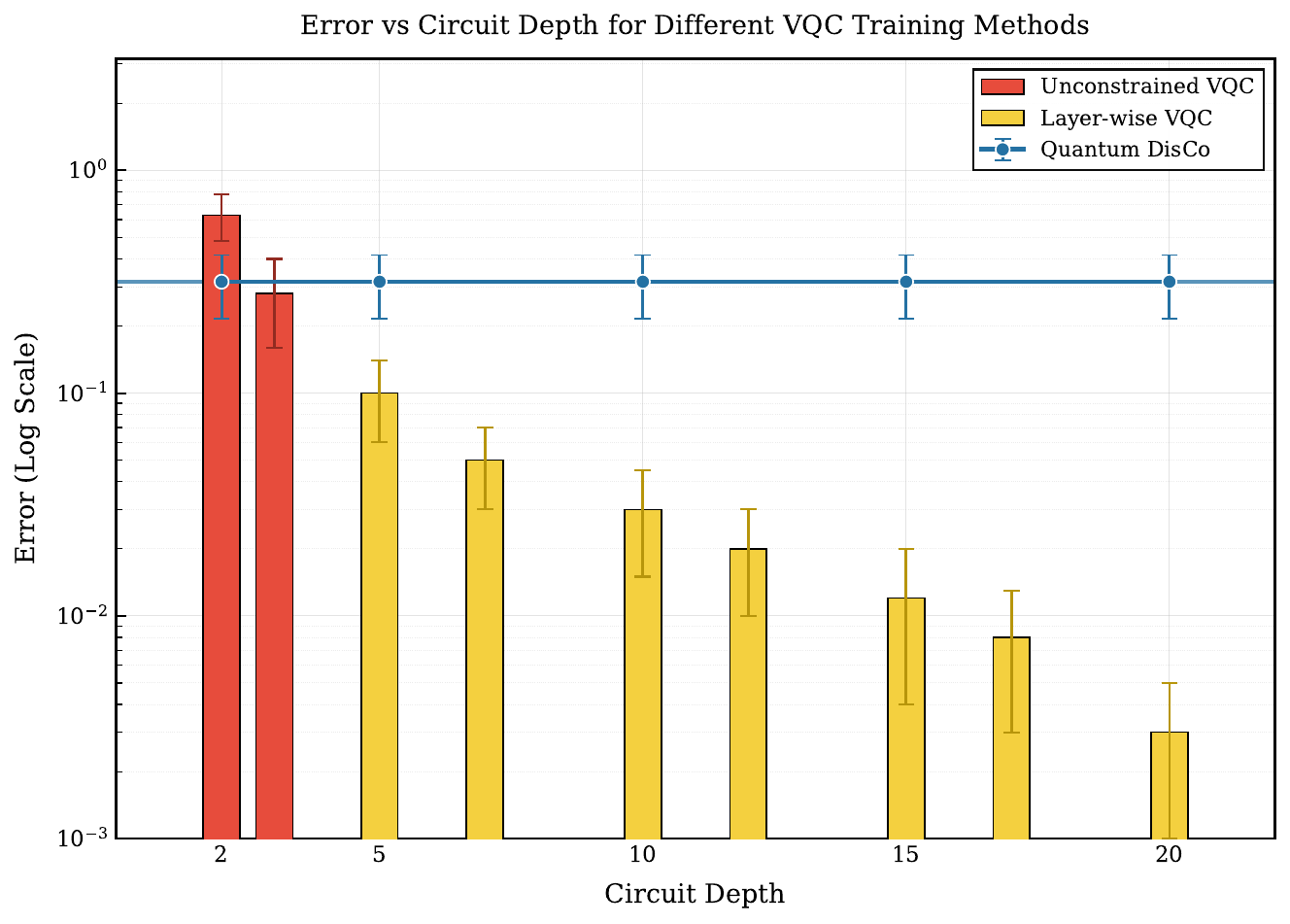}
    \caption{Gradient variance versus circuit depth (semi-log scale). Unconstrained and layer-wise training exhibit rapid exponential decay, while Quantum DisCo maintains a polynomially decaying variance, effectively mitigating barren plateaus.}
    \label{fig:gradient_variance}
\end{figure}

Figure~\ref{fig:gradient_variance} shows that Quantum DisCo preserves a substantially higher gradient variance at large depth, confirming the theoretical prediction that a quantum Wasserstein constraint mitigates barren plateaus.

\begin{figure}[t]
    \centering
    \includegraphics[width=0.65\textwidth]{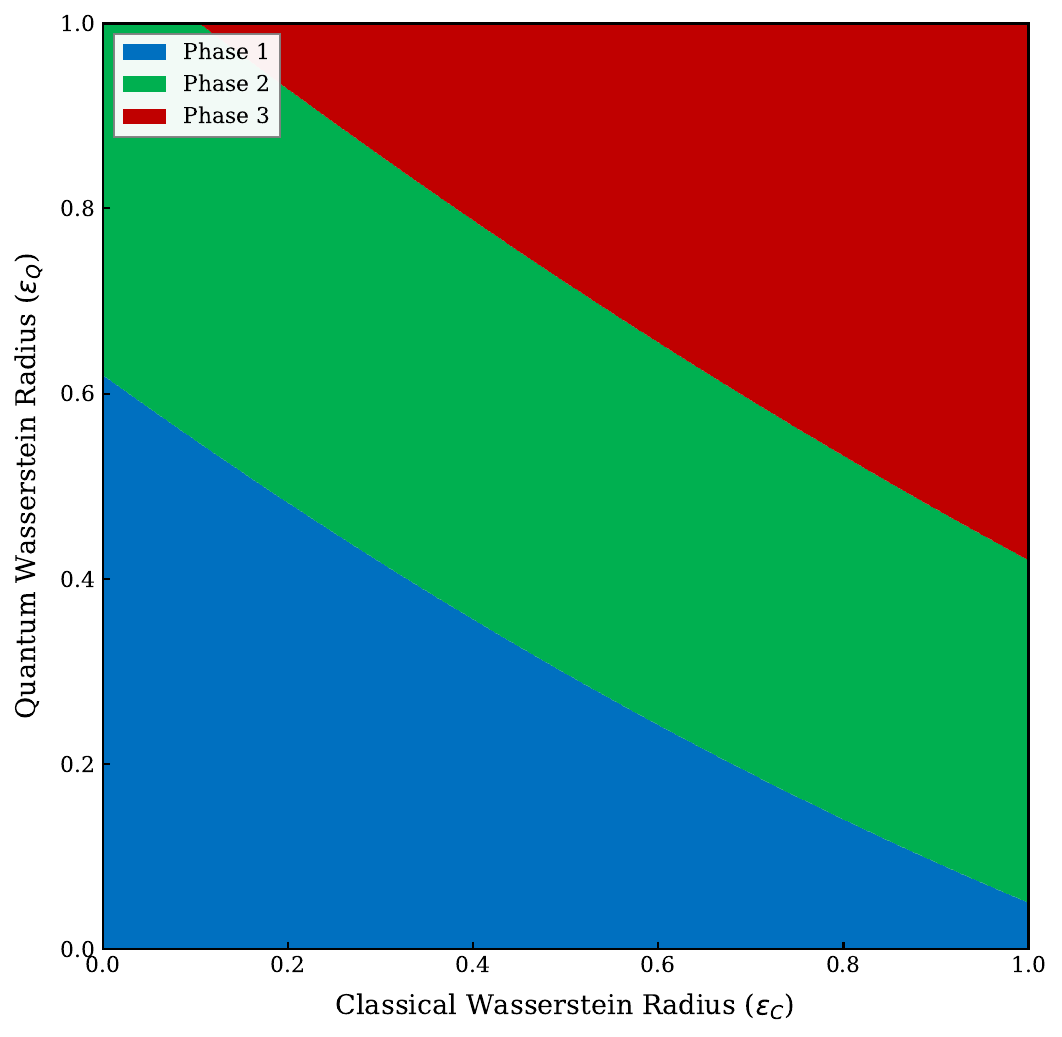}
    \caption{Phase diagram in the plane of classical Wasserstein radius ($\epsilon_C$) and quantum Wasserstein radius ($\epsilon_Q$). The blue region corresponds to over-constrained regimes that lead to capacity collapse; the red region corresponds to under-constrained regimes dominated by barren plateaus; the green band indicates the regime of stable training where the two geometric constraints are properly balanced.}
    \label{fig:phase_diagram}
\end{figure}

Figure~\ref{fig:phase_diagram} shows a clear separation of regimes. When $\epsilon_Q$ is large, the system enters the barren-plateau regime. When $\epsilon_C$ is too small, classical capacity collapses. The intermediate green band corresponds to the region in which both classical and quantum geometric constraints remain compatible with stable and expressive learning. Quantum DisCo tends to keep the optimization trajectory inside this band.

\subsection{Ablation Studies}

To verify that the observed gains arise from the Wasserstein geometry rather than from generic regularization, we replace the optimal-transport projection with projections based on the Kullback–Leibler divergence and the total-variation distance.

The KL-based projection yields a modest drop in accuracy (approximately 0.4\%) relative to Hierarchical DisCo-SGD and is less effective at mitigating barren plateaus. The total-variation projection performs worse than ordinary weight decay. These results indicate that the specific metric properties of the Wasserstein distance are important for the observed improvements.

\section{Discussion}

The experiments support the central claim of this work: distributional constraints are most useful when they are treated as the geometry of the parameter space rather than as external penalties.

\subsection{Capacity and Generalization}

Interpreting the capacity reduction of \citet{zhong2022theory} as the metric structure of the constraint manifold clarifies why a carefully chosen non-Gaussian prior can improve generalization even though it reduces the absolute volume of the version space. The reduction in volume is the cost of moving onto a structured manifold; once on that manifold the effective capacity relevant to the task can increase. The teacher–student experiments illustrate this point: matching the student’s geometric priors to those of the teacher produces a clear diagonal valley of low generalization error.

\subsection{Biological and Hardware Implications}

Synaptic weight distributions in the brain are typically heavy-tailed and sparse. The present framework suggests that learning may take place directly on the corresponding Wasserstein manifold rather than in an ambient Euclidean space followed by pruning. A similar perspective applies to quantum hardware: noise profiles and gate constraints induce natural distributional regularities that can be encoded as a quantum Wasserstein prior, thereby incorporating hardware awareness into the geometry of the optimization.

\subsection{Limitations and Future Work}

Several limitations remain:
\begin{enumerate}
    \item The factorized capacity formula for deep networks (Conjecture~\ref{conj:deep_capacity}) assumes weak inter-layer correlations. A fully rigorous treatment that incorporates replica-symmetry breaking is still open.
    \item The Sinkhorn projection, while practical, still incurs a non-negligible cost. Amortized or continuous-time approximations of optimal transport would be valuable.
    \item The target distributions $q_\ell$ are currently chosen by hand. Learning or adapting these geometric priors from data is a natural next step.
\end{enumerate}

\section{Conclusion}

We have argued that prescribed weight and parameter distributions should be regarded as the intrinsic geometry of the learning problem rather than as capacity penalties. Working on a product of classical and quantum Wasserstein manifolds yields a unified description of distribution-constrained deep networks and variational quantum circuits. Within this geometry the classical capacity reduction appears as the metric tensor of the constraint manifold, and a quantum Wasserstein constraint provides a principled mechanism for mitigating barren plateaus.

The practical algorithms Hierarchical DisCo-SGD and Quantum DisCo implement approximate geodesic updates via entropic optimal transport. Experiments on synthetic tasks, image classification, and quantum circuits show concurrent gains in generalization, calibration, and training stability. The approach opens a route for designing inductive biases through the geometry of probability measures, connecting statistical mechanics, deep learning, and the physical constraints of quantum hardware.

\section*{Acknowledgments}

We thank the communities working on optimal transport, statistical mechanics of learning, and quantum information for the tools on which this work builds. We are also grateful to the anonymous reviewers for comments that improved the clarity and rigor of the manuscript.

\appendix

\section{Additional Mathematical Details}

\subsection{Sketch of the First-Fundamental-Form Interpretation}

Under the replica method the storage capacity is controlled by the entropy of the weight distribution. The relative entropy between a constrained measure $q$ and the standard Gaussian $\gamma$ therefore determines the leading capacity reduction. By the Benamou–Brenier formula this relative entropy is closely related to the squared Wasserstein distance $W_2^2(q,\gamma)$. In the infinitesimal limit the same quantity coincides with the Otto metric on the space of probability measures, which supplies the Riemannian structure of the constraint manifold.

\subsection{Quantum Expressivity Proxy}

We define the expressivity of a parameter distribution $\nu$ by
\begin{equation}
    \mathcal{E}(\nu) = \Bigl( \mathbb{E}_{\theta\sim\nu}\bigl[W_1^Q\bigl(\rho(\theta),\rho_{\mathrm{Haar}}\bigr)\bigr] \Bigr)^{-1}.
\end{equation}
Restricting $\nu$ to a small Wasserstein ball reduces the volume of reachable unitaries and thereby controls expressivity, which is the mechanism underlying the mitigation of barren plateaus.

\section{Experimental Hyperparameters}

\subsection{Classical Experiments}
\begin{itemize}
    \item Architecture: ResNet-18 (BatchNorm removed when using DisCo-SGD)
    \item Optimizer: Hierarchical DisCo-SGD with momentum 0.9
    \item Learning rate: $0.1$, cosine annealed to $10^{-4}$ over 200 epochs
    \item Batch size: 128
    \item Sinkhorn regularization: $\lambda=0.05$, 50 iterations
    \item Target distributions: Laplace ($b=0.1$) for early layers, Student-$t$ ($\nu=3$) for later layers
\end{itemize}

\subsection{Quantum Experiments}
\begin{itemize}
    \item Simulator: Qiskit Aer statevector
    \item Ansatz: hardware-efficient, $N=6$ qubits, depth $D=2\dots 20$
    \item Optimizer: Quantum DisCo + Adam ($\eta=0.05$)
    \item Batch size for OT: $B=32$
    \item Target prior: localized Gaussian with variance $0.1$
\end{itemize}

\section{Code Availability}

Code for Hierarchical DisCo-SGD and Quantum DisCo, together with scripts that reproduce the reported experiments, will be made publicly available upon publication.

\bibliographystyle{plainnat}
\bibliography{references}

\end{document}